%% file: root.tex
\documentclass[12pt, final]{l4dc2023}

\title[Policy Learning for Active Target Tracking over Continuous $SE(3)$ Trajectories]{Policy Learning for Active Target Tracking over Continuous $SE(3)$ Trajectories}
\usepackage{times}

\DeclareMathOperator*{\vech}{vech}
\usepackage{algorithm2e}
\usepackage{tabularx,booktabs}
\usepackage{graphicx}
\usepackage{xcolor}
\usepackage{soul}
\usepackage{units}
\usepackage{footnote}
\usepackage{multirow}
\usepackage{fancyvrb} 
\VerbatimFootnotes    

\newcommand{\ith}{^{(i)}}

\input{sym.tex}

\input{sym_koga.tex}

\newtheorem{prop}{Proposition}

\author{%
 \Name{Pengzhi Yang} \Email{peyang@ucsd.edu}\\
 \Name{Shumon Koga} \Email{skoga@ucsd.edu}\\
 \Name{Arash Asgharivaskasi} \Email{aasghari@ucsd.edu}\\
 \Name{Nikolay Atanasov} \Email{natanasov@ucsd.edu}\\
 \addr Department of Electrical and Computer Engineering, University of California San Diego, La Jolla, CA 92093%
}

\begin{document}

\maketitle


\begin{abstract}%
This paper develops a \emph{model-based policy gradient algorithm} for tracking dynamic targets using a mobile agent equipped with an onboard sensor with limited field of view. The task is to obtain a continuous control policy for the mobile agent to collect sensor measurements that reduce uncertainty in the target states, measured by the target distribution entropy. We design a neural network control policy with the agent $SE(3)$ pose and the mean vector and information matrix of the joint target distribution as inputs and attention layers to handle variable numbers of targets. We also derive the gradient of the target entropy with respect to the network parameters explicitly, allowing efficient model-based policy gradient optimization.
\end{abstract}






\VerbatimFootnotes


\subsection*{Supplementary Material}
Open-source implementation: \href{https://github.com/ExistentialRobotics/RL_Active_Multi_Target_Tracking}{github.com/ExistentialRobotics/RL\_Active\_Multi\_Target\_Tracking}


\input{tex/Intro.tex}
\input{tex/Problem.tex}

\input{tex/Method.tex}
\input{tex/Experiments.tex}
\input{tex/Conclusion.tex}

\acks{We gratefully acknowledge support from NSF FRR CAREER 2045945 and ARL DCIST CRA W911NF17-2-0181.}

\bibliography{BIB_L4DC23.bib}

\end{document}

%% file: sym.tex
\newcommand{\calF}{{\cal F}}

\newcommand{\calI}{{\cal I}}

\newcommand{\calN}{{\cal N}}

\newcommand{\bfe}{\mathbf{e}}

\newcommand{\bfp}{\mathbf{p}}
\newcommand{\bfq}{\mathbf{q}}

\newcommand{\bfs}{\mathbf{s}}

\newcommand{\bfu}{\mathbf{u}}
\newcommand{\bfv}{\mathbf{v}}
\newcommand{\bfw}{\mathbf{w}}
\newcommand{\bfx}{\mathbf{x}}
\newcommand{\bfy}{\mathbf{y}}
\newcommand{\bfz}{\mathbf{z}}

\newcommand{\bfzeta}{\boldsymbol{\zeta}}
\newcommand{\bfeta}{\boldsymbol{\eta}}
\newcommand{\bftheta}{\boldsymbol{\theta}}

\newcommand{\bfmu}{\boldsymbol{\mu}}

\newcommand{\bfpi}{\boldsymbol{\pi}}

\newcommand{\bfomega}{\boldsymbol{\omega}}
\newcommand{\bfxi}{\boldsymbol{\xi}}

\newcommand{\bbH}{\mathbb{H}}

\newcommand{\bbR}{\mathbb{R}}
\newcommand{\bbS}{\mathbb{S}}

%% file: sym_koga.tex
\def\fr{\frac}
\def\pa{\partial}

\def\ba{\begin{align}}

\def\cal{\mathcal}

\def\R{\mathbb R}

\DeclareMathOperator*{\tr}{tr}

%% file: tex/Intro.tex
\section{Introduction}


Active target tracking is a problem in which the trajectory of a sensing agent is planned to reduce uncertainty in the state of a dynamic target of interest. This problem is motivated by several applications, including search and rescue \citep{kumar2004robot}, security and surveillance \citep{grocholsky2006cooperative}, wildfire detection \citep{julian2019distributed}, and pursuit evasion \citep{chung2011search}. Active information gathering in Simultaneous Localization and Mapping (SLAM) \citep{cadena2016past, placed2022survey} is an example of active target tracking in which the target is the (static) map of the environment. The challenge of the general active target tracking problem is inherent in predicting the future target state, optimizing the sensing agent trajectory with a limited Field of View (FoV), and taking into account the stochasticity of the target motion and sensor observations.

While the general active target tracking problem is posed as a stochastic optimal control problem due to the probabilistic inference of the target states, some earlier works have reduced this complexity. Under the assumption of linear Gaussian target motion and sensor observation models, active target tracking with an information-theoretic cost results in a deterministic optimal control problem, as shown in \cite{le2009trajectory}. \cite{Atanasov14ICRA} proposed a computationally efficient non-myopic planning approach with a strong performance guarantee even under a long planning horizon. \cite{schlotfeldt2019maximum} developed a consistent heuristic for applying A$^*$ search to the active information acquisition by deriving maximum upper bounds for the information measure. A multi-agent multi-target formulation of active information acquisition and associated scalable algorithms were studied by \cite{atanasov2015decentralized, schlotfeldt2018anytime, kantaros2019asymptotically, cai2021non}. While those works consider planning over discrete control space, \cite{koga2021active} proposed ``iterative Covariance Regulation'' (iCR), which optimizes the sensing trajectory over continuous $SE(3)$ space by deriving an analytical gradient of the cost with respect to the multi-step control sequence. Extensions of the work to occlusion-aware planning and to active SLAM under uncertain agent state was developed by \cite{asgharivaskasi2022active} and \cite{koga2022active}, respectively. However, all the aforementioned works compute control inputs for a given environment and cannot be applied to a new environment without replanning. 

Learning a control policy from training data obtained over several environments has been studied in the context of reinforcement learning (RL) \citep{sutton2018reinforcement}. RL methods employing deep neural network representations of the policy and value functions have been developed for both discrete control spaces \citep{mnih2015human} applied to games and continuous control spaces \citep{lillicrap2015continuous,schulman2017proximal} applied to robotics tasks. Learning a policy for active target tracking was proposed by \cite{jeong2019learning} using $Q$-learning to maximize the mutual information between the sensor data and the target states. \cite{hsu2021scalable} developed a multi-agent version of \cite{jeong2019learning} by incorporating an attention-block in the Q-network architecture. In addition, \cite{tang2021sensory} also leveraged an attention mechanism to achieve permutation-invariance in multi-agent settings. \cite{chen2020autonomous} focused on active landmark mapping using a graph neural network representing the exploration policy,
which is trained by Q-learning within a framework of Expectation Maximization \citep{wang2020autonomous}. \cite{chaplot2020learning} proposed
a modular and hierarchical approach to obtain a local policy by imitation learning from analytical path planners with a learned SLAM module and a global policy to maximize area coverage. \cite{lodel2022look} applied PPO \citep{schulman2017proximal} for learning an information-theoretic active mapping policy to acquire reference viewpoints that maximize reward with local sensing of obstacles and the agent position. \cite{yang2022learning} proposed a continuous trajectory learning method for active perception to localize multiple static landmarks, utilizing differentiable field of view for reward shaping and an attention-based neural network architecture. Learning low-level continuous control (e.g., velocity or torque) for $SE(3)$ agent kinematics using model-free RL methods is challenging because obtaining stable policy convergence requires a sufficiently large amount of experience, especially for complicated tasks. 



Utilizing a known or predicted state transition model in learning algorithms can significantly reduce the required amount of samples and computation relative to model-free RL methods. \cite{levine2013guided} developed a guided policy search that optimizes the system trajectory associated with the model by Differential Dynamic Programming (DDP) to achieve direct policy learning in control. Several variants and extensions of guided policy search were proposed by \cite{levine2014learning} for policy learning with unknown dynamics and by \cite{levine2016end} to obtain an end-to-end policy from visual sensing to robot action. \cite{luo2019reinforcement} incorporated a force and torque model into RL to enable high-precision robot manipulation tasks. A recent comprehensive review of the model-based RL was presented by \cite{janner2019trust}.

\textbf{Contributions:} The contributions of the paper are summarized as follows.
\begin{itemize}
    \item We develop a novel model-based policy gradient algorithm for tracking multiple dynamic targets over continuous $SE(3)$ trajectories. A differentiable field-of-view (FoV) formulation is incorporated to enable offline learning for sensor models with limited FoV. 
    
    \item We design a neural network policy architecture with an attention block to handle multiple targets and with padding and masking to enable learning over a varying number of targets during training.  
\end{itemize}

%% file: tex/Problem.tex
\section{Problem Statement}

Consider an agent with pose $T_k \in SE(3) \subset \bbR^{4 \times 4}$ at time $t_k \in \R_+$, where $\{t_k\}_{k=0}^{K}$ for some $K \in {\mathbb N}$ is an increasing sequence. The definition of pose and its discrete-time kinematic model are:
\begin{align} 
  T_k:= \left[ 
\begin{array}{cc} 
R_k & \bfx_k \\
{\mathbf 0}_{3 \times 1}^\top & 1 
\end{array} 
\right],  \quad \label{SE2dyn} 
 T_{k+1} =  T_{k} \exp \left(\tau_k \hat \bfu_k \right) , 
\end{align}  
where $\bfx_k \in  \R^3$ is position, $R_k \in SO(3) \subset \bbR^{3 \times 3} $ is orientation, $\tau_k := t_{k+1} - t_k > 0$ is the sampling-time interval, and $\bfu_k = [\bfv_k^\top, \bfomega_k^\top]^\top \in \R^6$ is a control input, consisting of linear velocity $\bfv_k \in \R^3$ and angular velocity $\bfomega_k \in \R^3$. The hat operator $ \hat{(\cdot)} : \R^6 \to se(3)$ maps vectors in $\R^6$ to the Lie algebra $se(3)$ associated with the $SE(3)$ Lie group \citep{barfoot2017state}. 
 
We consider a finite number of moving targets $\bfy_k = [\bfy^{(1)}_k, \dots, \bfy^{(n_l)}_k]$, where $\bfy^{(j)}_k \in \bbR^{n_y}$ for $j \in \{1, \dots, n_l\}$ denotes the $n_y$-dimensional state of $j$-th target at time $k$ and $n_l$ is the total number of targets. We assume that each target has homogeneous dynamics governed by a linear Gaussian process:
\begin{equation} \label{eq:target-dyn}
    \bfy^{(j)}_{k+1} = A \bfy^{(j)}_k + B \bfxi_k^{(j)} + \bfw_k^{(j)},   
\end{equation}
where $A: \bbR^{n_y \times n_y}$ and $B: \bbR^{n_y \times m_y}$ are the system matrices, $\bfxi_k^{(j)} \in \bbR^{m_y}$ is a known target input, and $ \bfw_{k}^{(j)} \sim \calN(0, W_k)$ is a stochastic process noise assumed to be Gaussian with zero mean and covariance $W_k \in \bbR^{n_y \times n_y}$. 

The agent is equipped with an onboard sensor for tracking the target states. Let $\calF \subset \bbR^3$ represent the FoV of the sensor within the agent's body frame. The set of target indices within the FoV is:
\begin{align} \label{eq:IF} 
\calI_{{\calF}}(T, \{\bfy^{(j)}\}) = \left\{ j \in \{1, \dots, n_l\} \mid   \bfq \left(T, \bfzeta(\bfy^{(j)}) \right)  \in \calF \right\}, 
\end{align}
where $\bfzeta: \bbR^{n_y} \to \bbR^3$ transforms the target state to the 3-D coordinate of the target's location, and $\bfq: SE(3) \times \bbR^3 \to \bbR^3 $ returns the agent-body-frame coordinates of $\bfzeta \in \bbR^3$ given by 
\begin{align}
    \bfq(T, \bfzeta) = Q T^{-1} \underline{\bfzeta} , \label{eq:qj-def}
\end{align}
where the projection matrix $Q$ and the homogeneous coordinates $\underline{\bfzeta}$ are defined as:
\begin{align} 
  Q = \left[ 
  \begin{array}{cc} 
  I_3 & {\mathbf 0}_{3 \times 1} 
  \end{array} 
  \right] \in \bbR^{3 \times 4}, \;\;
  \underline{\bfzeta}  = \left[ 
  \begin{array}{c} 
  \bfzeta \\ 1 
  \end{array} 
  \right]  \in \bbR^4.  
\end{align}
Then, a sensor measurement is denoted by $\bfz_k = [ 
    \{ \bfz_k^{(j)}\}_{j \in \calI_{{\calF}}(T_k, \{\bfy^{(j)}_k\})}
    ]\in \bbR^{n_z | \calI_{{\calF}}(T_k, \{\bfy^{(j)}_k\})|}$ where $\bfz_k^{(j)} \in \bbR^{n_z}$ is an observation of $j$-th target with model:
    \begin{align}
    \bfz_k^{(j)} &=  H \bfy_k^{(j)} +  \bfeta_{k}, \quad \bfeta_k \sim \calN (0, V),  \label{eq:sensor}
\end{align}
for all $j \in  \calI_{{\calF}}(T_k, \{\bfy^{(j)}_k\}),$ where the matrix $H \in  \bbR^{n_z \times n_y}$ is the sensor matrix and $V \in \bbR^{n_z \times n_z}$ is the sensing noise covariance.

Our task is to develop a control policy for the agent to minimize uncertainty about the multiple targets using information acquired from the onboard sensor. We consider minimizing the differential entropy
\footnote{The differential entropy of a continuous random variable $Y$ with probability density function $p$ is defined as $\bbH(Y) := - \int p(y) \log p(y) dy$.} 
$\bbH(\bfy_K| \bfz_{0:K},  T_{0:K})$ of the terminal target state $\bfy_K$ given a sequence of sensor observations $\bfz_{0:K}$ and the agent trajectory $T_{0:K}$. Since each target state is independent of all  other target states due to the independent motion model in \eqref{eq:target-dyn}, the problem is equivalent to 
\begin{align}
    \min \sum_{j=1}^{n_l} \bbH(\bfy^{(j)}_K | \bfz_{0:K},  T_{0:K}) . \label{eq:info-cost}
\end{align}
Under the Gaussian target state obeying \eqref{eq:target-dyn} with the linear Gaussian sensor model \eqref{eq:sensor}, the problem \eqref{eq:info-cost} is equivalent to
\begin{align}
    \max \sum_{j=1}^{n_l} \log \det \left( Y_K^{(j)} \right), 
\end{align}
where $Y_K^{(j)}$ is the terminal information matrix of the posterior distribution of target state $\bfy^{(j)}_K$. More precisely, we denote the prior and posterior distributions of the target state given a history of measurements as:
\begin{align}
    \bfy_{k}^{(j)} | \bfz_{0:k-1} \sim \calN(\bfp_k^{(j)}, (P_k^{(j)})^{-1}), \quad \bfy_k^{(j)} | \bfz_{0:k} \sim \calN(\bfmu_k^{(j)}, (Y_k^{(j)})^{-1}) , 
\end{align}
for all $j \in \{1, \dots, n_l\}$ and $k \in \{1, \dots, K\}$. The mean and covariance (or information) matrix are updated based on the Kalman Filter, which is given by the following prediction and update steps \citep{Atanasov14ICRA} (here we omit the superscripts $^{(j)}$ to ease the notation but the variables are for each $j$-th target):
\begin{align}
\textrm{ \textbf{Prediction} (for all $j \in \{1, \dots, n_l\}$):} \quad 
\bfp_{k+1} & = A \bfmu_k + B \bfxi_k, \\
P_{k+1} & = (A Y_k^{-1} A^\top + W_k)^{-1}, \\
\textrm{\textbf{Update} (for $j \in  \calI_{{\calF}}(T_{k+1}, \{\bfy^{(j)}_{k+1}\})$:} \quad 
\bfmu_{k+1} &= \bfp_{k+1} + K_{k+1} (\bfz_{k+1} - H \bfp_{k+1}) , \label{eq:KF-FoV}\\ 
    Y_{k+1} &=  P_{k+1} + H^\top V^{-1} H, \\
K_{k+1} & = P_{k+1}^{-1} H_{k+1}^\top ( H_{k+1} P_{k+1}^{-1} H_{k+1}^{\top} + V_{k+1})^{-1}. 
\end{align}
However, since the update step is performed only for targets within FoV, which is known only after obtaining the sensing with limited FoV, the implementation above 
is not possible in the offline planning stage. To enable planning before measurements are obtained, following \cite{koga2021active}, we introduce a differentiable FoV formulation to relax the index condition and enable gradient computation. Moreover, during training, we suppose that the sensor noise is negligible to enable offline non-myopic planning without acquiring measurements, thereby leading to identical prior and posterior means. We then design the control policy $\bfu_k = \bfpi_{\bftheta}(\bfs_k)$ as a deep neural network, where $\bfs_k$ is the input of the network. Since the differentiable FoV renders the posterior information matrix at next time step dependent on the prior mean and information matrix of the target state, the input of the network is designed to include them. Finally, we consider the following problem for policy optimization. 

\textbf{Problem} \label{prob:AM-policy}
Given a prior Gaussian distribution for moving target $\bfy^{(j)} \sim \calN(\bfp_0^{(j)}, (P_0^{(j)})^{-1})$ with mean $\bfmu_0^{(j)} \in \R^{n_y}$ and information matrix $Y_0^{(j)} \in \bbS^{n_y\times n_y}_{\succ 0}$ for all $j \in \{1, \dots, n_l\}$, optimize the parameters $\bftheta \in \R^{n_p}$ of a control policy $\bfu_k = \bfpi_{\bftheta}(\bfs_k)$ where $\bfs_k = [ \textrm{log}\left( T_k \right)^{\vee}, \{\bfp_{k+1}^{(j)}, \vech(P_{k+1}^{(j)})\}_{j=1}^{n_l}]$ by solve the following policy optimization problem:
\begin{align} 
\max_{\bftheta \in \bbR^{n_p}} \hspace{1mm} \sum_{j=1}^{n_l} \log \det (Y_{K}^{(j)}), \label{prob:reward-policy}
\end{align} 
subject to 
\begin{align} \label{prob:model-policy} 
T_{k+1} & = T_k \exp (\tau \bfpi_{\bftheta}(\bfs_k)) \\
    \bfp^{(j)}_{k+1} & = A \bfp^{(j)}_k + B \bfxi_k^{(j)}, \label{eq:mean-update} \\
    P_{k+1}^{(j)} & = ( A (Y_k^{(j)})^{-1} A^\top + W_k)^{-1}, \label{eq:prediction-update} \\
    Y_{k+1}^{(j)} & = P_{k+1}^{(j)} + M (T_{k+1}, \bfp_{k+1}^{(j)}) , \label{eq:info-update}\\
    M(T, \bfp^{(j)}) &= \left(  1 - \Phi( d(\bfq(T, \bfp^{(j)}), {\mathcal F})) \right) H^\top V^{-1} H,\label{eq:M} 
\end{align} 
for all $j \in \{1, \dots, n_l\}$ and $k \in \{0, \dots, K-1\}$, where \eqref{eq:M} is derived in \cite{koga2021active}, $\Phi$ is a probit function \citep{bishop2006pattern}, defined by the Gaussian CDF $\Phi : \bbR \to [0, 1]$,  $ \Phi(x) = \frac{1}{2} \left[ 1 + \textrm{erf} \left(\frac{x}{\sqrt{2} \kappa} - 2  \right) \right] $, and $d$ is a signed distance function associated with the FoV $\calF$ defined below.

\begin{definition}
The \emph{signed distance function} $d : \R^3 \to \R $ associated with a set $\calF \subset \bbR^3$ is:
\begin{align}
     d(\bfq, {\mathcal F})  = 
     \begin{cases}
     - \min_{\bfq^* \in \pa {\mathcal F}}  || \bfq - \bfq^*||, \quad \textrm{if} \quad \bfq \in {\mathcal F}, \\
     \phantom{-} \min_{\bfq^* \in \pa {\mathcal F}} ||\bfq - \bfq^*||, \quad \textrm{if} \quad \bfq \notin {\mathcal F},  
     \end{cases}
\end{align}
where $\pa {\mathcal F}$ is the boundary of ${\mathcal F}$.
\end{definition}

%% file: tex/Method.tex
\section{Model-Based RL over Continuous $SE(3)$ Trajectory}

We approach the problem in the previous section by the following steps. First, we derive the gradient of the cumulative reward with respect to the policy parameters analytically by utilizing the $SE(3)$ pose kinematics and the mean and information update, similarly to iCR \citep{koga2021active}. Then, we design a neural network architecture to handle multiple targets and to enable learning over a varying number of targets.

\subsection{Analytical Policy Gradient}

The following proposition provides an update rule for the policy function parameters $\bftheta$ using the gradient of the reward function in \eqref{prob:reward-policy} with respect to $\bftheta$.


\begin{prop} \label{prop:gd-policy}
The gradient-ascent update for solving active exploration \eqref{prob:reward-policy}--\eqref{eq:info-update} with differentiable field of view \eqref{eq:M} is given by
\begin{align}
\theta^{(i+1)} =  \theta\ith + \gamma\ith \frac{\partial r^{\bfpi_{\bftheta}}}{\partial \theta\ith}, 
\end{align}
where $\gamma\ith \in \bbR_+$ for all $i \in \{1, \dots, n_p\}$ is a step size, and the gradient is given by 
\begin{align} 
\fr{\pa r^{\bfpi_{\bftheta}}}{\pa \theta^{(i)}} 
 &= \sum_{j=1}^{n_l} \tr \left( (Y_K^{(j)})^{-1} \Omega_K^{(j,i)} \right) ,  \label{log-grad-policy} 
\end{align}
where 
\begin{align}\label{eq:Omega-def}
    \Omega^{(j,i)}_k := \frac{\partial Y^{(j)}_k}{\partial \theta^{(i)}} \in \bbR^{n_y \times n_y}, \qquad \Lambda^{(i)}_k := \frac{\partial T_k}{\partial \theta^{(i)}} \in \bbR^{4 \times 4}
\end{align}
are obtained via:
  %
  \begin{align} \label{eq:Lambda-ini-policy} 
  \Lambda^{(i)}_{0}    &=  0, \quad \Omega^{(j,i)}_0 = 0, \quad \forall i \in \{1, \dots, n_p \}, \quad \forall j \in \{1, \dots, n_l\} \\
   \Omega^{(j,i)}_{k+1} &=  (  A^\top + Y_k^{(j)} A^{-1} W_k)^{-1} \Omega_k^{(j,i)} ( A + W_k A^{ -\top } Y_k^{(j)})^{-1} \notag\\
   & + \left( \Phi'( d(\bfq, {\mathcal F})) \frac{\partial d}{\partial \bfq} \right) \bigg|_{\bfq = \bfq(T_{k+1}, \bfp_{k+1}^{(j)})} Q T_{k+1}^{-1} \Lambda^{(i)}_{k+1} T_{k+1}^{-1} \underline{\bfp_{k+1}^{(j)}} H^\top V^{-1} H, \label{eq:Omega-update} \\
     \Lambda^{(i)}_{k+1} &=  \Lambda_k^{(i)} \exp (\tau_k \pi_{\bftheta}(\bfs_k)) + T_{k} \sum_{j=1}^{6} \bfe_{6,j}^\top \frac{\partial \bfpi_{\bftheta} (\bfs_k)}{\partial \bftheta} \bfe_{n_p,i} \frac{\partial \exp( \tau_k \hat \bfu )}{\partial \bfu^{(j)} } \bigg|_{\bfu = \bfpi_{\bftheta} (\bfs_k)}, \label{eq:Lambda-update}
  \end{align}
  where $\bfe_{n,m} \in \bbR^{n}$ is a $n$-dimensional unit vector whose $m$-th element is $1$ and all others are $0$.  
\end{prop}

\begin{proof}
    Taking the gradient of the reward \eqref{prob:reward-policy} directly leads to \eqref{log-grad-policy} by defining $\Omega_k^{(j,i)}$ as \eqref{eq:Omega-def}, which is the perturbation of the information matrix with respect to the policy parameter. Then, the update equation \eqref{eq:Omega-update} is derived by taking the gradient of both sides in \eqref{eq:info-update} and in \eqref{eq:prediction-update}. The same can be performed for the $SE(3)$ pose state to derive \eqref{eq:Lambda-update} from the gradient of \eqref{prob:model-policy}. Note that the prior mean is not affected by the control policy due to the update equation \eqref{eq:mean-update} and, hence, we do not need to define the perturbation of the prior mean.   
\end{proof}

\begin{figure*}
  \centering
  \includegraphics[width=\linewidth]{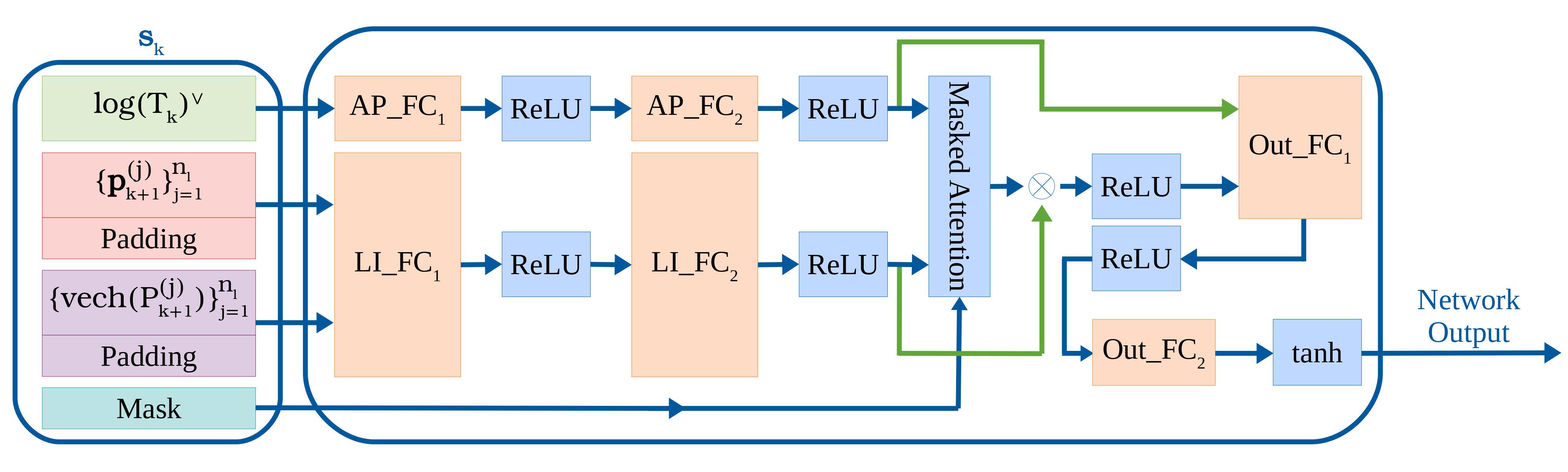}
  \caption{Deep neural network architecture used for the parameterized policy $\bfpi_{\bftheta}(\bfs_k)$. The input $\bfs_k$ contains the current agent pose in log representation $\log{(T_k)}^{\vee}$ as well as zero-padded predicted target states $\{\bfp_{k+1}^{(j)}\}_{j=1}^{n_l}$ and target information $\{P_{k+1}^{(j)}\}_{j=1}^{n_l}$. The \textit{Mask} vector indicates which elements of the padded target input contains relevant values, allowing to remove the influence of zero-padding in the final output. For each input, the network computes continuous controls $\bfu_k = [\bfv_k^\top, \bfomega_k^\top]^\top$.}
  \label{fig:network}
\end{figure*}

\subsection{Network Architecture}

In order to generalize training and testing to varying number of targets, we utilize a padding and masking scheme that allows tracking an arbitrary number of targets (up to a defined maximum $n_l^{\text{max}}$) while keeping the network architecture unchanged. We consider a fixed-length input vector with $n_l^{\text{max}} \times n_y$ elements for both the target state and target information, where only the first $n_l \times n_y$ elements contain non-zero values. Additionally, a binary \textit{Mask} vector contains instruction about which elements of the subsequent computations should be ignored to cancel out the effect of padding values in the output of the network. Fig.~\ref{fig:network} illustrates the policy network architecture using the padding and masking scheme, where similarly to \cite{yang2022learning} we employ an attention mechanism \citep{long2020evolutionary} so that the agent takes into account the relationship between its current pose state and the moving target states in order to prioritize observing uncertain targets. The fully-connected layers \textit{AP\_FC} and \textit{LI\_FC} alongside the \textit{ReLU} nonlinearities compute embeddings for the agent pose and target information, denoted as $\text{Emb}_a$ and $\text{Emb}_l$ respectively. The \textit{Masked Attention} block blends information from the agent, targets, and the masking as follows:
\begin{align}
    \label{eq:masking}
    \textit{Masked Attention}(\text{Emb}_a, \text{Emb}_l, \text{Mask}) = \textit{softmax}\left((1 + \log{(\text{Mask})}) \odot \left(\frac{\text{Emb}_a \text{Emb}_b^\top}{\alpha}\right) \right),
\end{align}
where the operator $\odot$ denotes to element-wise vector multiplication and $\alpha$ is a network hyper-parameter. As \eqref{eq:masking} shows, the components affected by padding can be nullified via the element-wise multiplication since the softmax operator eliminates the influence of the components corresponding to the zero elements of \textit{Mask} in the subsequent matrix multiplication with $\text{Emb}_l$. Therefore, for any input vector $\bfs_k$ with an arbitrary choice of targets, the policy network $\bfpi_{\bftheta}(\bfs_k)$ computes continuous controls $\bfu_k = [\bfv_k^\top, \bfomega_k^\top]^\top$.

%% file: tex/Experiments.tex
\section{Experiments}
In this section, we examine the performance of both our proposed model-based RL and a benchmark model-free RL for active target tracking.  
We provide simulation results to visualize the tracking trajectories and quantitative comparison results of the reward value to demonstrate the robust performance of our model-based RL method.

\subsection{Experiment Settings}
In the evaluation, we consider 2-D target tracking using a ground vehicle governed by $SE(2)$ differential-drive kinematics and 2-D target positions as the target states ($n_y = 2$). As in practice, we control only the agent's forward and angular speeds within a limited range for linear velocity $v_x \in [0, 4]$ m/s and angular velocity $\omega \in [-\pi/3, \pi/3]$ rad/s. The agent
is equipped with a sensor detecting the relative position from the agent to the targets within a triangular FoV with depth of $2$ m and angular range of $2\pi/3$ rad/s. When the targets are inside the agent's FoV, their estimated position is updated based on \eqref{eq:KF-FoV}.

We set the matrices in the target model as $A = I$ and $B = I$. Regarding the known input $\bfxi \in \bbR^2$,  we tested two cases: \par
1) \textbf{Unbiased motion:} The known input $\bfxi$ is sampled from a uniform distribution as $\bfxi \sim \textrm{Uniform}[-\bar \xi, \bar \xi]$, i.e., the mean velocities are $0$ and hence the target motion is unbiased. The targets move within small areas based on their current position at every episode. \par
2) \textbf{Biased motion:} The absolute mean of the uniform distribution is set to larger than $0$, i.e., the targets have a base linear velocity heading to the same direction with small randomness at each time step. With this setting, the targets do not move too far from other targets during tracking. \par
We trained the neural network policy only from the biased target motion, because the mean of the uniform distribution is also sampled from the uniform distribution with zero mean, which includes the case of unbiased motion. Besides, the same hyper-parameters, environment settings, and network architecture were applied for model-free and model-based training. Specifically, regarding the environment, the smoothing factor $\kappa$, the magnitude of the Gaussian sensor noise $\sigma_1$, and the magnitude of the Gaussian motion noise $\sigma_2$ were set with values of $0.4$, $0.2$, and $0.05$, respectively. Besides, the time horizon and the targets' initialization position boundaries were set to the same constants dependent on the number of targets which is varied between $[3, 8]$ at each episode during training. For the policy network, the fully-connected layers $\textit{AP\_FC}_1$, $\textit{AP\_FC}_2$, $\textit{LI\_FC}_1$, $\textit{LI\_FC}_2$, $\textit{Out\_FC}_1$, and $\textit{Out\_FC}_2$ have $32$, $32$, $64$, $32$, $64$ and $2$ units, respectively. We choose a hyper-parameter $\alpha = 4$ throughout our experiments. For the model-free reinforcement learning baseline, we apply PPO \citep{schulman2017proximal} for training, while the model-based policy was directly trained using gradient ascent over a batch of the last $20$ episodes of an epoch, without a replay buffer.

\begin{table}[h]
\normalsize
\centering
\caption{{Comparison of the proposed model-based RL with the model-free RL.} The table shows the average and standard derivation for normalized rewards.}
\resizebox{.9\columnwidth}{!}{%
    \begin{tabular}{ccccc}
        \toprule
          \multirow{2}{*}{Method} &
          \multirow{2}{*}{Target Motion Model} &
          $\textrm{3 Targets}$ & $\textrm{5 Targets}$ & $\textrm{7 Targets}$ \\\cline{3-5}
           & & Episodic Reward & Episodic Reward & Episodic Reward \\
        \hline
         \multirow{2}{*}{\textit{Model-free RL}} & Unbiased Motion & 4.57 $\pm$ 2.13 & 3.65 $\pm$ 1.54 & 1.94 $\pm$ 1.70 \\
         & Biased Motion & 4.23 $\pm$ 1.77  & 3.27 $\pm$ 1.30 & 2.01 $\pm$ 1.67 
         
         \\ \hline
         \multirow{2}{*}{\textit{Model-based RL}} & Unbiased Motion & \textbf{6.71 $\pm$ 1.47} & \textbf{6.56 $\pm$ 0.55} & \textbf{5.41 $\pm$ 0.85} \\
         & Biased Motion & \textbf{6.87 $\pm$ 1.21} & \textbf{5.96 $\pm$ 0.96} & \textbf{4.92 $\pm$ 1.07}
         \\
        \bottomrule
    \end{tabular}}
\label{Quantitative_Results}
\end{table}

\subsection{Comparison Results and Analysis}
As shown in Fig.~\ref{fig:trajectories}, we compare the trajectories generated by model-free and model-based trained networks in three scenarios of $3$, $5$, and $7$ targets. In each scenario, the targets' initial position and velocity and the agent's initial pose are identical for fair comparisons. We can clearly see that the network trained with the model-based algorithm is capable of controlling the agent in a better manner for target tracking, while the network trained with the model-free algorithm renders the agent prone to move in a small area.

Quantitative comparisons are shown in Table \ref{Quantitative_Results}. The metric of the episodic reward is computed based on \eqref{prob:reward-policy}. We normalized the value by dividing by the number of targets at the end of each episode. We chose three random seeds $0$, $10$, and $100$ for both algorithms and tested all the models with two target motions in each scenario. With each setting in one scenario, both methods were tested for $30$ runs. According to Table \ref{Quantitative_Results}, model-based RL has an overall better performance in terms of the larger episodic reward with smaller variance for both unbiased and biased target motion. It is also observed that the average reward is inversely proportional to the number of targets for both methods. We conjecture that when the number of targets increases while the size of the map is also larger, efficient planning for target tracking becomes more challenging accordingly.

\begin{figure*}[t]
  \centering
  \resizebox{1\textwidth}{!}{
  \begin{tabular}{ccccc}
      Num of Targets & Step 0 & $\frac{1}{3}$Horizon & $\frac{2}{3}$Horizon & End of the Episode\\
     3 Targets & 
     \multicolumn{1}{m{5cm}}{\includegraphics[height=2in, trim={0.5cm .4cm .2cm .cm}, clip]{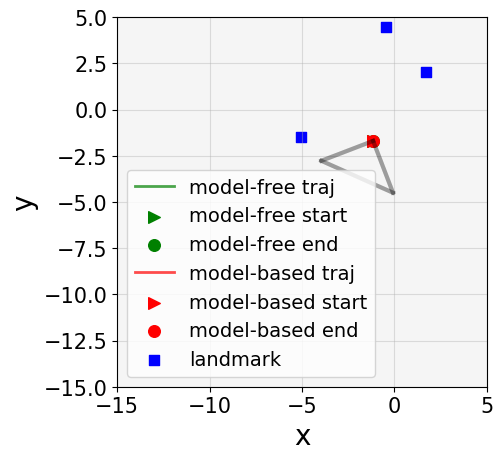}} & 
     \multicolumn{1}{m{5cm}}{\includegraphics[height=2in, trim={1cm .4cm .2cm .2cm}, clip]{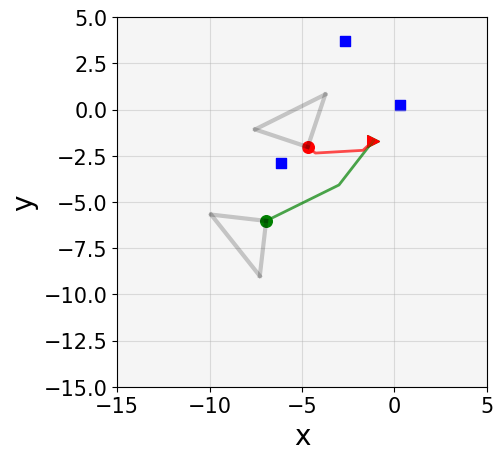}} & 
     \multicolumn{1}{m{5cm}}{\includegraphics[height=2in, trim={1cm .4cm .2cm .2cm}, clip]{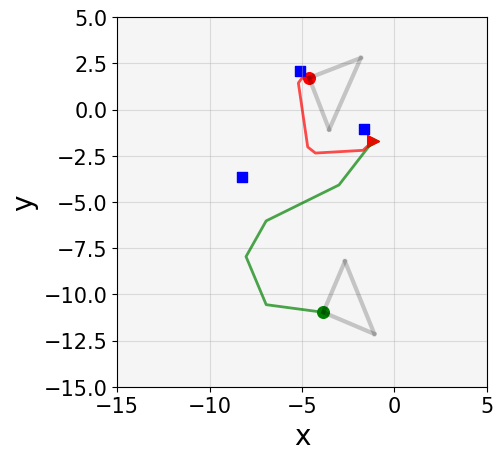}} & 
     \multicolumn{1}{m{5cm}}{\includegraphics[height=2in, trim={1cm .4cm .2cm .2cm}, clip]{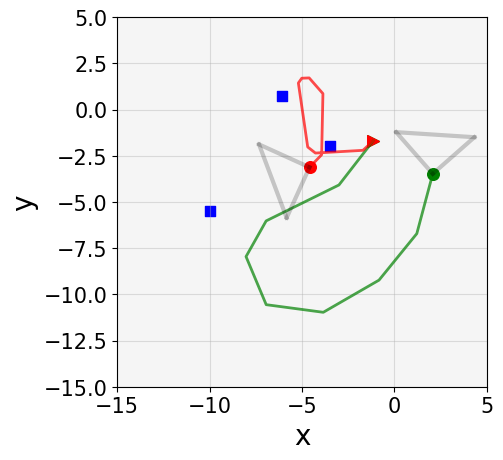}}\\
     5 Targets & 
     \multicolumn{1}{m{5cm}}{\includegraphics[height=2in, trim={0.5cm .4cm .2cm .2cm}, clip]{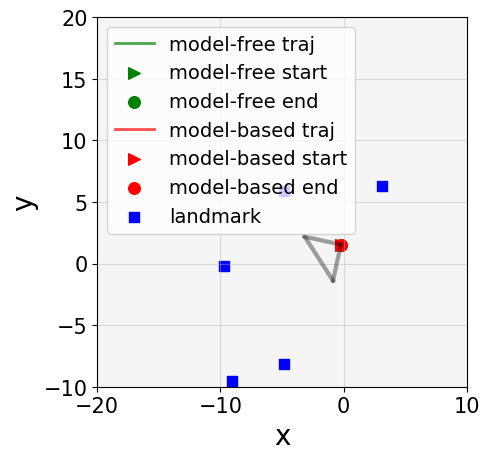}} & 
     \multicolumn{1}{m{5cm}}{\includegraphics[height=2in, trim={1cm .4cm .2cm .2cm}, clip]{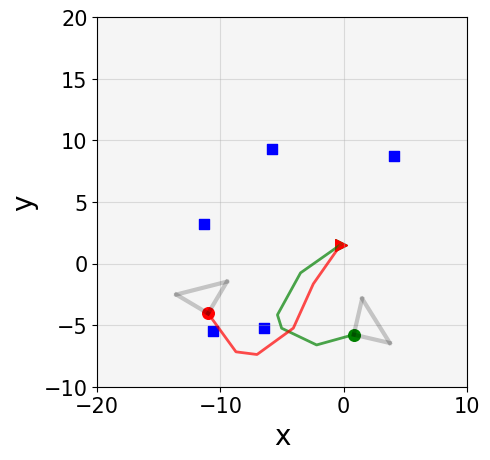}} & 
     \multicolumn{1}{m{5cm}}{\includegraphics[height=2in, trim={1cm .4cm .2cm .2cm}, clip]{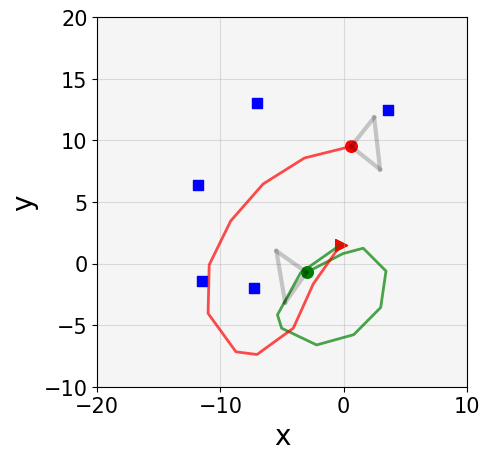}} & 
     \multicolumn{1}{m{5cm}}{\includegraphics[height=2in, trim={1cm .4cm .2cm .2cm}, clip]{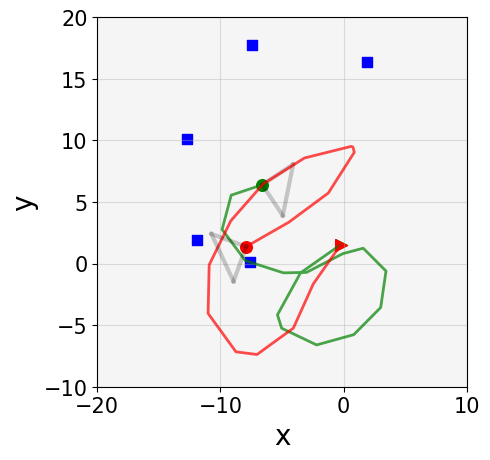}}\\
     7 Targets & 
     \multicolumn{1}{m{5cm}}{\includegraphics[height=2in, trim={0.5cm .4cm .2cm .2cm}, clip]{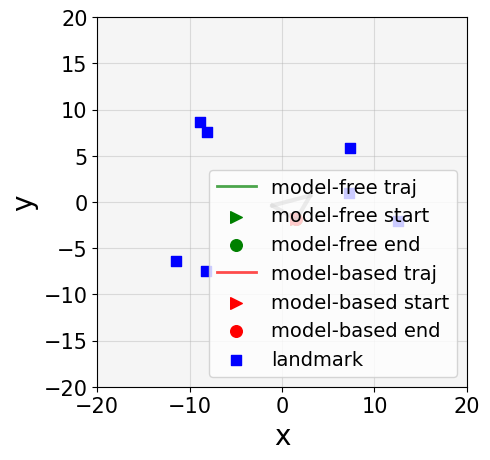}} & 
     \multicolumn{1}{m{5cm}}{\includegraphics[height=2in, trim={1cm .4cm .2cm .2cm}, clip]{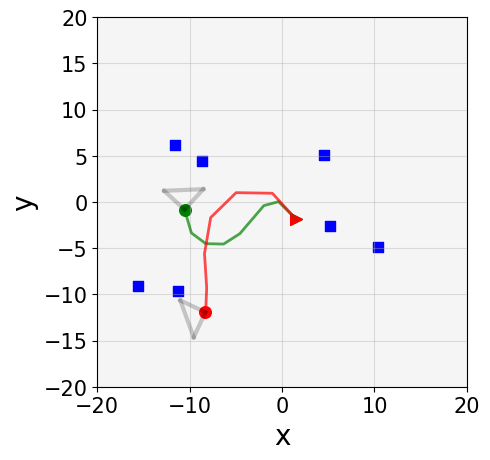}} & 
     \multicolumn{1}{m{5cm}}{\includegraphics[height=2in, trim={1cm .4cm .2cm .2cm}, clip]{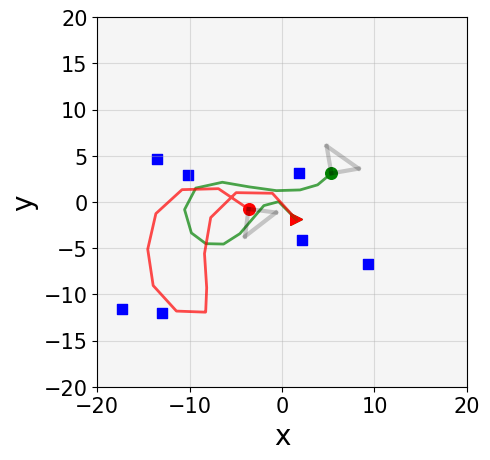}} & 
     \multicolumn{1}{m{5cm}}{\includegraphics[height=2in, trim={1cm .4cm .2cm .2cm}, clip]{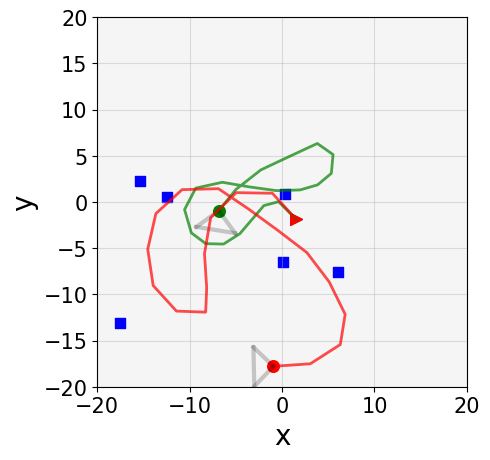}}\\
     
  \end{tabular}}

  \caption{Sensing agent trajectories for target tracking. Two methods are compared in three scenarios with $3$, $5$, and $7$ targets. The blue squares represent the moving targets. The green and red curves show the sensing trajectories generated by model-free and model-based trained networks. The grey triangles illustrate the agents' forward triangular field of view.}
  \label{fig:trajectories}
\end{figure*}

%% file: tex/Conclusion.tex
\section{Conclusion}

This paper proposed a model-based reinforcement learning algorithm for tracking multiple dynamic targets using a mobile agent with limited FoV. The prior and posterior mean and information matrix of each target state were obtained by Kalman filtering. We derived an analytical gradient of the target entropy cost function with respect to the parameters of the control policy network by introducing a differentiable FoV and using perturbation of the $SE(3)$ state and the information matrix to obtain a continuous control policy. We observed that our model-based RL algorithm achieves better multi-target tracking in a simulated environment than a model-free RL algorithm based on proximal policy optimization. In future research, we will consider learning the policy for an unknown number of targets, in the presence of obstacles in the environment, and for target tracking by a team of agents with limited communication.